\documentclass{article}

\usepackage{fullpage}

\usepackage[utf8]{inputenc} 
\usepackage[T1]{fontenc}    
\usepackage{hyperref}       
\usepackage{url}            
\usepackage{booktabs}       
\usepackage{amsfonts}       
\usepackage{nicefrac}       
\usepackage{microtype}      
\usepackage{xcolor}         

\usepackage{amsthm,amsfonts,amssymb,amsmath}
\usepackage{xspace}
\usepackage{bbm}
\usepackage{graphicx}
\usepackage{bbm}

\usepackage{scrextend}

\usepackage{accents}

\usepackage{graphicx} 
\usepackage{booktabs} 

\usepackage{caption}

\usepackage{xcolor}

\usepackage{pgfplots}
\usepackage{hyperref}

\usepackage{comment}
\usepackage{caption}
\usepackage{subcaption}

\pgfplotsset{compat=1.10}

\usepackage{tikz}

\usepackage{enumitem}
\usepackage[ruled,vlined]{algorithm2e}
\usepackage[noend]{algorithmic}

\usepackage{natbib}
\usepackage{thmtools}
\usepackage{thm-restate}

\usetikzlibrary{quotes,positioning}

\newtheorem{theorem}{Theorem}[section]
\newtheorem{lemma}{Lemma}[section]
\newtheorem{remark}{Remark}[section]

\newtheorem{definition}{Definition}[section]

\newtheorem{fact}{Fact}[section]

\usepackage{xcolor}

\newcommand{\E}{\mathop{\mathbb{E}}}
\newcommand{\cX}{\mathcal{X}}
\newcommand{\cY}{\mathcal{Y}}
\newcommand{\cZ}{\mathcal{Z}}
\newcommand{\cT}{\mathcal{T}}
\newcommand{\cA}{\mathcal{A}}
\newcommand{\cO}{\mathcal{O}}
\newcommand{\cR}{\mathcal{R}}

\newcommand{\ind}{\mathbbm{1}}

\newcommand{\updreq}{\mathtt{UpdReq}}

\newcommand{\D}{\mathcal{D}}

\newcommand{\cP}{\mathcal{P}}

\newcommand{\R}{\mathcal{R}}

\newcommand{\Z}{\mathcal{Z}}
\newcommand{\A}{\mathcal{A}}

\renewcommand{\P}{\mathcal{P}}

\newcommand{\publish}{f_{\text{publish}}}

\newcommand{\vertiii}[1]{{\left\vert\kern-0.25ex\left\vert\kern-0.25ex\left\vert #1
    \right\vert\kern-0.25ex\right\vert\kern-0.25ex\right\vert}}

\newcommand{\stateSpace}{\mathcal{S}}
\newcommand{\state}{s}

\newcommand{\cAdistr}{\cA^{\text{distr}}}
\newcommand{\cAsingle}{\cA^{\text{single}}}

\newcommand{\Sampler}{\mathtt{Sampler}}

\newcommand{\ppredict}{\mathtt{PrivatePredict}}

\usepackage{amsmath}
\usepackage{pgfplots}
\usepackage{authblk}

\title{Adaptive Machine Unlearning}

 \author[1]{Varun Gupta}
 \author[1]{Christopher Jung}
 \author[2]{Seth Neel}
 \author[1]{\authorcr Aaron Roth}
 \author[1]{Saeed Sharifi-Malvajerdi}
 \author[3]{Chris Waites}
 \affil[1]{University of Pennsylvania}
 \affil[2]{Harvard University}
 \affil[3]{Stanford University}

\begin{document}
\maketitle
\begin{abstract}
    Data deletion algorithms aim to remove the influence of deleted data points from trained models at a cheaper computational cost than fully retraining those models. However, for sequences of deletions, most prior work in the non-convex setting gives valid guarantees  only for sequences that are chosen \emph{independently} of the models that are published. If people choose to delete their data as a function of the published models (because they don't like what the models reveal about them, for example), then the update sequence is \emph{adaptive}. In this paper, we give a general reduction from deletion guarantees against adaptive sequences to deletion guarantees against non-adaptive sequences, using differential privacy and its connection to max information. Combined with ideas from prior work which give guarantees for non-adaptive deletion sequences, this leads to extremely flexible algorithms able to handle arbitrary model classes and training methodologies, giving strong provable deletion guarantees for adaptive  deletion sequences. We show in theory how prior work for non-convex models fails against adaptive deletion sequences, and use this intuition to design a practical attack against the SISA algorithm of \cite{unlearning} on CIFAR-10, MNIST, Fashion-MNIST. 
    \end{abstract}

\section{Introduction}
Businesses like Facebook and Google depend  on training sophisticated models on user data. Increasingly---in part because of regulations like the European Union’s General Data
Protection Act and the California Consumer Privacy Act---these organizations are receiving requests to delete the data of particular users. But what should that mean? It is straightforward to delete a customer's data from a database and stop using it  to train \emph{future} models. But what about models that have already been trained using an individual's data? These are not necessarily safe; it is known that individual training data can be exfiltrated from models trained in standard ways via 
 \textit{model inversion} attacks \citep{attack, veale, inversion}. Regulators are still grappling with  when a trained model should be considered to contain personal data of individuals in the training set and the potential legal implications. In  2020 draft guidance, the U.K.'s Information Commissioner's Office addressed how to comply with data deletion requests as they pertain to ML models: 
\begin{quote}
\textit{If the request is for rectification or erasure of the data, this may not be possible without re-training the model...or deleting the model altogether} \citep{ico}.
\end{quote}

Fully retraining the model  every time a deletion request is received  can be prohibitive in terms of both time and money---especially for large models and frequent deletion requests. The problem of \emph{data deletion} (also known as \emph{machine unlearning}) is to find an algorithmic middle ground between the compliant but impractical baseline of retraining, and the potentially illegal standard of doing nothing. We iteratively update  models as deletion requests come in, with the twin goals of having computational cost that is substantially less than the cost of full retraining, and the guarantee that the models we produce are (almost) indistinguishable from the models that would have resulted from full retraining.

After an initial model is deployed deletion requests arrive  over time as users make  decisions about whether to delete their data. It is easy to see how these decisions may be \textit{adaptive} with respect to the models. For example, security researchers may publish a new model inversion attack that identifies a specific subset of people in the training data, thus leading to increased deletion requests for people in that subset. In this paper we give the first machine unlearning algorithms that both have rigorous deletion guarantees against these kind of adaptive deletion sequence, and can accommodate arbitrary non-convex models like deep neural networks without requiring pretraining on non-user data. 

\subsection{Main Results}
The deletion guarantees proven for several prior methods crucially rely on the implicit assumption that the points that are deleted are independent of the randomness used to train the models. However this assumption fails unless the sequence of deletion requests is chosen  independently of the information that the model provider has made public. This is a very strong assumption, because users may wish to delete their data \emph{exactly because of what  deployed models reveal about them}.

We  give a generic reduction. We show that if:
\begin{enumerate}
    \item A data deletion algorithm $\cR_\cA$ for a learning algorithm $\cA$ has deletion guarantees for \emph{oblivious}  sequences of deletion requests (as those from past work do), and
    \item Information about the internal randomness of $\cR_\cA$ is revealed only in a manner that satisfies \emph{differential privacy}, then
\end{enumerate}
($\cA,\cR_\cA)$ \emph{also satisfies data deletion guarantees against an adaptive sequence of deletion requests}, that can depend in arbitrary ways on the information that the model provider has made public.

This generic reduction can be used to give adaptive data deletion mechanisms for a wide variety of problems by leveraging past work on deletion algorithms for non-adaptive sequences, and a  line of  work on differentially private aggregation \citep{papernot2018scalable, privatepred}. Since prior deletion algorithms themselves tend to use existing learning algorithms in a black-box way, the entire pipeline  is modular and easy to bolt-on to existing methods. In Section~\ref{sec:distributed}, we show how this can be accomplished by using a variant of the SISA framework of \cite{unlearning} together with a differentially private aggregation method. 

In Section~\ref{sec:experiment}, we complement our main result with a theoretical example and a set of experimental results on CIFAR-10, MNIST, and Fashion-MNIST  that serve to illustrate two points:
\begin{enumerate}
\item  Past method's lack of guarantees for adaptive sequences is not simply a failure of analysis, but an actual failure of these methods to satisfy deletion guarantees for adaptive deletion sequences. As an exemplar, we use two variants of SISA from \cite{unlearning} that both satisfy perfect deletion guarantees for non-adaptive deletion sequences and exhibit adaptive deletion sequences that strongly separate the resulting distribution on models compared to the retraining baseline.
\item  That differential privacy may be useful in giving adaptive guarantees beyond the statement of our theorems. Specifically we show that small amounts of noise addition (insufficient for our theorems to apply) already serve to break the adaptive deletion strategies that we use to  falsify the adaptive deletion guarantees in our experiments described in point 1.
\end{enumerate}

\subsection{Related Work}
Data deletion was introduced by \cite{CY15}; we adopt the randomized formulation of \cite{forgetu}. \cite{forgetu} anticipate the problem of deletion requests that might be correlated with internal state of the algorithm, and define (and propose as a study for future work) \emph{robust} data deletion which is a data deletion guarantee that holds for adversaries with knowledge of the internal state. Our insight is that we can provide  deletion guarantees against adaptive sequences by instead obscuring the internal state of the algorithm using techniques from differential privacy. 

We are the first  to explicitly consider the problem of adaptive sequences of deletion requests, but some techniques from past work \emph{do} have deletion guarantees that extend to adaptive sequences.  Deterministic methods and methods that depend only on randomness that is  sampled after the deletion request are already robust to adaptive deletion. This includes techniques that find an approximately optimal solution to a strongly convex problem and then perturb the solution to obscure the optimizer within a small radius e.g.  \cite{hessian, neel2021descent,sekhari2021remember}. It also includes the approach of \cite{Golatkar,golatkar2020forgetting} which pre-trains a nonconvex model on data that  will never be deleted and then does convex fine-tuning on user data on top of that. Techniques whose deletion guarantees depend on randomness sampled at training in general do not have guarantees against adaptive deletions. This includes algorithms given in \cite{forgetu,unlearning,neel2021descent} --- the SISA framework of \cite{unlearning} being of particular interest as it is  agnostic to the class of models and training methodology, and so is extremely flexible. 

Differential privacy has been used as a mitigation for adaptivity  since the work of \cite{DFHPRRstoc,DFHPRRscience}. In machine learning, it has been used to mitigate the bias of adaptive data gathering strategies as used in bandit learning algorithms \citep{NR18}.
The application that is most similar to our work is \cite{streaming}, which uses differential privacy of the internal randomness of an algorithm (as we do) to reduce streaming algorithms with  guarantees against adaptive adversarial streams to streaming algorithms with guarantees against oblivious adversaries. Our techniques differ; while \cite{streaming} reduce to the so-called ``transfer theorem for linear and low sensitivity queries'' developed over a series of works \cite{DFHPRRstoc,BNSSSU16,JLNRSSS20}, we use a more general connection between differential privacy and ``max-information'' established in \cite{dwork2015adaptive,bounded-max-info}.

\section{Preliminaries}
\label{sec:prelim}
Let $\cZ$ be the data domain. A dataset $D$ is a multi-set of elements from $\cZ$. We consider update requests of two types: deletion and addition. These update requests are formally defined below, similar to how they are defined in \citep{neel2021descent}.

\begin{definition}[Update Operations and Sequences]{}
An update $u$ is a pair $(z, \bullet)$ where $z \in \cZ$ is a datapoint and $\bullet \in \cT = \{ \mathtt{'add'}, \mathtt{'delete'} \}$ determines the type of the 
update. An update sequence $U$ is a sequence $(u^1, u^2, \ldots)$ where $u^t \in \cZ \times \cT$ for all $t$. Given a dataset $D$ and an update $u = (z, \bullet)$, the update operation is defined as:
\begin{align*}
    D \circ u \triangleq 
    \begin{cases}
        D \cup \{z\} & \text{if } \bullet = \mathtt{'add'} \\
        D \setminus \{z\} & \text{if } \bullet = \mathtt{'delete'} \\
   \end{cases}
\end{align*}
Given an update sequence $U = (u^1, u^2, \ldots)$, we have $D \circ U \triangleq (((D \circ u^1) \circ u^2) \circ \ldots )$.
\end{definition}

We use $\Theta$ to denote the space of models. A \textit{learning} or \textit{training} algorithm is a mapping $\cA : \cZ^* \rightarrow \Theta^*$ that maps a dataset $D \in \cZ^*$ to a collection of models $\theta \in \Theta^*$. An \textit{unlearning} or \textit{update} algorithm for $\cA$ is a mapping $\cR_{\cA}: \cZ^* \times (\cZ \times \cT)  \times \stateSpace \rightarrow \Theta^*$ which takes in a data set $D \in \cZ^*$, an update request $u \in \cZ \times \cT$, and some current state for the algorithms $\state \in \stateSpace$ (the domain $\stateSpace$ can be arbitrary), and outputs an updated collection of models $\theta' \in \Theta^*$. 
In this paper we consider a setting in which a stream of update requests arrive in sequence.  We note that in this sequential framework, the update algorithm $\cR_\cA$ also updates the state of the algorithm after each update request is processed; however, for notational economy, we do not explicitly write the updated state as an output of the algorithm.

At each round, we provide access to the models through a mapping $\publish^t: \Theta^* \to \Psi$  that takes in the collection of models and outputs some object $\psi \in \Psi$. A published object $\psi \in \Psi$ can, for instance, be the aggregate predictions of the learned models on a data set, or, some aggregation of the models. To model adaptively chosen update sequences, we define an arbitrary ``update requester'' who interacts with the learning and unlearning algorithms $(\A, \R_\A)$ through the publishing function $\publish$ in rounds to generate a sequence of updates. The update requester is denoted by $\updreq$ and defined  in Definition~\ref{def:updreq}, and the interaction between the algorithms and the update requester is described in Algorithm~\ref{alg:interaction}.

Throughout we will use $u^t$ to denote the update request at round $t$. We will use $D^t$ to denote the data set at round $t$: $D^0$ is the initial training data set and for all $t \ge 1$, $D^t = D^{t-1} \circ u^t$. We will use $\theta^t$ to denote the learned models at round $t$: $\theta^0$ is generated by the initial training algorithm $\cA$, and $\theta^t$ for $t \ge 1$ denotes the updated models at round $t$ generated by the update algorithm $\cR_\cA$. $\psi^t$ denotes the published object at round $t$: $\psi^t = \publish^t (\theta^t)$.

\begin{definition}[Update Requester ($\updreq$)]\label{def:updreq}
The update sequence is generated by an update requester which is modeled by a (possibly randomized) mapping $\updreq: \Psi^* \times ( \cZ \times \cT)^* \to (\cZ \times \cT)$ that takes as input the history of interaction between herself and the algorithms, and outputs a new update for the current round. Given an update requester $\updreq$,  algorithms $(\A, \R_\A)$ and publishing functions $\{ \publish^t \}_t$, the update sequence $U=\{u^t\}_t$ can be written as
\[
u^1 = \updreq \left( \psi^0\right), \ u^2 = \updreq \left( \psi^0, u^1, \psi^1\right), \ldots, \ u^t = \updreq \left( \psi^0, u^1, \psi^1, \ldots, u^{t-1}, \psi^{t-1}\right)
\]
We say an update requester $\updreq$ is nonadaptive if it is independent of the published objects, i.e., if there exists a mapping $\updreq': ( \cZ \times \cT)^* \to (\cZ \times \cT)$ such that for all $t \ge 1$,
\[
u^t = \updreq \left( \psi^0, u^1, \psi^1, u^2, \ldots, u^{t-1}, \psi^{t-1}\right) = \updreq' \left( u^1, u^2, \ldots, u^{t-1}\right)
\]
This is equivalent to saying that the update sequence is fixed before the interaction occurs.
\end{definition}

\begin{algorithm}[t]
\SetAlgoLined 
\begin{algorithmic}[1]
\STATE \textbf{Input}:  Data set $D$
\STATE Let $D^0 \gets  D$.
\STATE Train $\theta^0 \gets \cA (D)$.
\STATE Publish $\psi^0 \gets \publish^0 (\theta^0)$.
\STATE Save the initial state $s^0$.
\FOR{$t = 1, 2, \dots$}
    \STATE The update requester requests a new update, given the history of interaction:
    \STATE \hspace{1cm} $u^t \gets \updreq \left( \psi^0, u^1, \psi^1, u^2, \ldots, u^{t-1}, \psi^{t-1}\right)$.
    \STATE The algorithms update, given $u^t$: 
    \STATE \hspace{1cm} Update the models $\theta^t \gets \cR_\cA \left(D^{t-1}, u^t, s^{t-1} \right)$.
    \STATE \hspace{1cm} Publish $\psi^t \gets \publish^t \left(\theta^t \right)$.
    \STATE \hspace{1cm} Save the updated state $s^{t}$.
    \STATE \hspace{1cm} Update the data set $D^{t} \gets D^{t-1} \circ u^t$.
\ENDFOR
\end{algorithmic}
\caption{Interaction between $(\cA, \cR_\cA)$  and $\updreq$}
\label{alg:interaction}
\end{algorithm}

Following  \citep{forgetu}, we propose the following definition for an unlearning algorithm in the sequential update setting (\citep{forgetu} gives a definition for a single deletion request, whereas here we define a natural extension for an arbitrarily long sequence of deletions, as well as additions, that can be chosen adaptively.). Informally, we require that at every round, and for all possible update requesters, with high probability over the draw of the update sequence, no subset of models resulting from deletion occurs with substantially higher probability than it would have under full retraining. 

\begin{definition}[$(\alpha, \beta, \gamma)$-unlearning]\label{def:unlearning}
We say that $\cR_{\cA}$ is an $(\alpha, \beta, \gamma)$-unlearning algorithm for $\cA$, if for
all datasets $D = D^0$ and all update requesters $\mathtt{UpdReq}$, the following condition holds:
For every update step $t \geq 1$, with probability at least $1-\gamma$ over the draw of the update sequence $u^{\le t} = (u^1, \ldots, u^t)$ from $\updreq$,
\begin{align*}
    \forall E \subseteq \Theta^*: \quad 
    \Pr \left[ \cR_\cA \left(D^{t-1}, u^{t}, s^{t-1} \right) \in E \, \middle\vert \, u^{\le t}\right]
    \le
    e^{\alpha} \cdot \Pr \left[ \cA \left( D^t \right) \in E  \right] + \beta
\end{align*}
We say $\cR_{\cA}$ is a nonadaptive $(\alpha, \beta, \gamma)$-unlearning algorithm for $\cA$ if the above condition holds for any nonadaptive $\updreq$.
\end{definition}

\begin{remark}
Our definition of unlearning is reminiscent of differential privacy, but following \citep{forgetu}, we ask only for a \emph{one-sided} guarantee: that the probability of any event under the unlearning scheme is not too much larger than the probability of the same event under full retraining, but not vice versa. The reason is that we do not want there to be events that can substantially increase an observer's confidence that we did \emph{not} engage in full retraining, but we do not object to observers who strongly update their beliefs that we \emph{did} engage in full retraining. Our events $E$ are defined directly over the sets of models in $\Theta^*$ output by $\cA$ and $\cR_{\cA}$ --- note that because of information processing inequalities, this is only stronger than defining events $E$ over the observable outcome space $\Psi$. 
\end{remark}

\subsection{Differential Privacy and Max-Information}
Differential privacy will be a key tool in our results. 
Let $\cX$ denote an arbitrary data domain. We use $x \in \cX$ to denote an individual element of $\cX$, and $X \in \cX^*$ to denote a collection of elements from $\cX$ --- which we call a data set. We say two data sets $X,X' \in \cX^*$ are neighboring if they differ in at most one element. We say an algorithm $M: \cX^n \to \cO$ is differentially private if its output distributions on neighboring data sets are close, formalized below.
\begin{definition}[Differential Privacy (DP) \citep{DMNS06,DKMMN06}]
An algorithm $M: \cX^m \to \cO$  is $(\epsilon, \delta)$-differentially private, if for every neighboring $X$ and $X'$, and for every $O \subseteq \cO$, we have
$\Pr \left[ M(X) \in O \right] \le e^{\epsilon} \Pr \left[ M(X') \in O \right] + \delta.$
\end{definition}
We remark at the outset that the ``datasets'' to which we will eventually ask for differential privacy with respect to will not be the datasets on which our learning algorithms are trained, but will instead be collections of random bits parameterizing our randomized algorithms.

Differentially private algorithms are robust to data-independent post-processing:
\begin{lemma}[Post-processing preserves DP \citep{DMNS06}]\label{lem:postprocessing}
If $M: \cX^m \to \cO$ is $(\epsilon, \delta)$-differentially private, then for all $f: \cO \to \mathcal{R}$, we have $f \circ M: \cX^m \to \mathcal{R}$ defined by $f \circ M (X) = f(M(X))$ is $(\epsilon, \delta)$-differentially private.
\end{lemma}

The max-information between two jointly distributed random variables measures how close their joint distribution is to the product of their corresponding marginal distributions.
\begin{definition}[Max-Information \citep{dwork2015adaptive}]\label{def:max-info}
Let $X$ and $Y$ be jointly distributed random variables over the domain $(\cX, \cY)$.
 The $\beta$-approximate max-information between $X$ and $Y$ is:
\begin{align*}
    I_{\infty}^\beta(X; Y) = \log \sup_{E \subseteq (\cX, \cY), \Pr[(X, Y) \in E] > \beta} \frac{\Pr[(X, Y) \in E] - \beta}{\Pr[(X \otimes Y) \in E]}
\end{align*}
where $(X \otimes Y)$ represents the product distribution of $X$ and $Y$.
\end{definition}

The max-information of an algorithm $M$ that takes a dataset $X$ as input and outputs $M(X)$, is defined as the max-information between $X$ and $M(X)$ for the worst case product distribution over $X$:
\begin{definition}[Max-Information of an Algorithm \citep{dwork2015adaptive}]
Let $M: \cX^m \to \cO$ be an Algorithm. We say $M$ has $\beta$-approximate max-information of $k$, written $I_\infty^\beta (M,m) \le k$, if for every distribution $\cP$ over $\cX$, we have $I_\infty^\beta (X; M(X)) \le k$ when $X \sim \cP^m$.
\end{definition}

In this paper, we will use the fact that differentially private algorithms have bounded max-information:
\begin{theorem}[DP implies bounded max-information \citep{bounded-max-info}]\label{thm:DPmaxinfo}
Let $M: \cX^m \rightarrow \cO$ be an $(\epsilon, \delta)$-differentially private algorithm for $0 < \epsilon \le 1/2$ and $0 < \delta < \epsilon$. Then,
$I_\infty^\beta (M,m) = O\left(\epsilon^2 m + m \sqrt{\delta / \epsilon}\right)$ for $\beta = e^{-\epsilon^2 m} + O\left(m\sqrt{\delta / \epsilon}\right)$.
\end{theorem}

\section{A Reduction from Adaptive to Nonadaptive Update Requesters}
\label{sec:reduction}
In our analysis we imagine without loss of generality that the learning algorithm $\cA$ draws an $i.i.d.$ sequence of random variables $r \sim \cP^m$ (that encodes all the randomness to be used over the course of the updates) from some distribution $\cP$, and passes it to the unlearning algorithm $\cR_\cA$. Note $r$ is drawn once in the initial training, and given $r$, $\cA$ and $\cR_\cA$ become deterministic mappings. We can also view  the state $s^t$ as a deterministic mapping of $r$, the update requests so far $u^{\le t} = (u^1, \ldots, u^{t})$, and the original data set $D^0$. We write $s^t = g^t (D^0, u^{\le t}, r)$ for some  deterministic mapping $g^t$. We can therefore summarize the trajectory of the algorithms $(\cA, \cR_\cA)$ as follows.
\begin{itemize}
    \item $t=0$: draw $r \sim \cP^m$, let $\theta^0 = \cA (D) \equiv \cA (D; r)$, and $\psi^0 = \publish^0 \left( \theta^0 \right)$.
    \item $t \ge 1$: $\theta^t = \cR_\cA (D^{t-1}, u^{t}, s^{t-1})$ where $s^{t-1} = g^{t-1} (D^0, u^{\le t-1}, r)$, and $\psi^t = \publish^t \left( \theta^t \right)$.
\end{itemize}

In this view, the  randomness $r$ used by the  learning algorithm $\cA$ and the subsequent invocations of the unlearning algorithm $\cR_{\cA}$ is represented as part of the internal state. Past analyses of unlearning algorithms have crucially assumed that $r$ is statistically independent of the updates $(u^1, u^2, \ldots)$ (which is the case for non-adaptive update requesters, but not for adaptive update requesters). In the following general theorem, we show that if a learning/unlearning pair satisfies unlearning guarantees against non-adaptive update requesters, and the publishing function is differentially private \emph{in the internal randomness} $r$, then the resulting algorithms also satisfy unlearning guarantees against adaptive update requesters. Note that what is important is that the publishing algorithms are differentially private in the \emph{internal randomness $r$}, not in the datapoints used for training.

\begin{restatable}[A General Theorem]{theorem}{thmGeneralTheorem}
\label{thm:general-theorem}
Fix a pair of learning and unlearning algorithms $(\cA, \cR_\cA)$ and the publishing functions $\{ \publish^t \}_t$. Suppose for every round $t$, the sequence of publishing functions $\{ \publish^{t'} \}_{t' \le t}$ is $(\epsilon,\delta)$-differentially private in $r \sim \cP^m$, for $0 < \epsilon \le 1/2$ and $0 < \delta < \epsilon$. Suppose $\cR_\cA$ is a non-adaptive $(\alpha, \beta, \gamma)$-unlearning algorithm for $\cA$. Then  $\cR_\cA$ is an $(\alpha', \beta', \gamma')$-unlearning algorithm for $\cA$ for 
$
\alpha' = \alpha + \epsilon',
\beta' = \beta e^{\epsilon'} + \sqrt{\delta'},
\gamma' = \gamma+ \sqrt{\delta'}
$
where $\epsilon' = O\left(\epsilon^2 m + m \sqrt{\delta / \epsilon}\right)$ and
$\delta' =  e^{-\epsilon^2 m} + O\left( m\sqrt{\delta / \epsilon}\right)$.
\end{restatable}

The proof can be found in the Appendix, but at an intuitive level, it proceeds as follows. Because it does not change the joint distribution on update requests and internal state, we can imagine in our analysis that $r$ is redrawn after each update request from its conditional distribution, conditioned on the observed update sequence so far. Because the publishing function is differentially private in $r$, by the fact that post-processing preserves differential privacy (Lemma~\ref{lem:postprocessing}), so is the update sequence. We may therefore apply the max-information bound (Theorem \ref{thm:DPmaxinfo}), which allows us to relate the conditional distribution on $r$ to its original (prior) distribution $\cP^m$. But resampling $r$ from $\cP^m$ removes the dependence between $r$ and the update sequence, which places us in the non-adaptive case, and allows us to apply the hypothesized unlearning guarantees for nonadaptive update requesters.

\section{Distributed Algorithms}\label{sec:distributed}
In this section, we describe a general family of distributed learning and unlearning algorithms that are in the spirit of the ``SISA'' framework of \cite{unlearning} (with one crucial modification). At a high level, the SISA framework operates by first randomly dividing the data into $k$ ``shards'', and separately training a model on each shard. When a new point is deleted, it is removed from the shards that contained it, and only the models corresponding to those shards are retrained. The flexibility of this methodology is that the models and training procedures used in each shard can be arbitrary, as can the aggregation done at the end to convert the resulting ensemble into predictions: however these choices are instantiated, this framework gives a $(0,0,0)$-unlearning algorithm against any non-adaptive update requester (Lemma~\ref{lem:non-adaptive}). Here we show that if the $k$ shards are selected \emph{independently} of one another, then we can apply our  reduction given in the previous section with $m = k$ and obtain algorithms that satisfy deletion guarantees against adaptive update requesters.

A \emph{distributed} learning algorithm $\cAdistr: \cZ^* \to \Theta^*$ is described by a single-shard learning algorithm $\cAsingle: \cZ^* \to \Theta$ and a routine $\Sampler$, used to select the points in a shard. $\Sampler$, given a dataset $D$ and some probability $p \in [0,1]$, includes each element of $D$ in the shard with probability $p$. 

Distributed learning algorithm $\cAdistr$ creates $k$ independent shards from the  dataset $D$ of size $n$ by running $\Sampler$ $k$ times and training a model with $\cAsingle$ on each shard $i \in [k]$ to form an ensemble of $k$ models. To emphasize that the randomness across shards is independent, we will instantiate $k$ independent samplers $\Sampler_i$ and training algorithms $\cAsingle_i$ for each shard $i \in [k]$. 
We formally describe $\cAdistr$ in Algorithm~\ref{alg:learning-distributed0}.

\begin{algorithm}[t]
\SetAlgoLined 
\begin{algorithmic}
\STATE \textbf{Input}: dataset $D \equiv D^0$ of size $n$
    \STATE Draw the shards: $D^0_i = \Sampler(D^0, p)$, for every $i \in [k]$.
    \STATE Train the models: $\theta^0_i = \cAsingle (D^0_i)$, for every $i \in [k]$.
\STATE Save the state: $s^0 = (\{D^0_i\}_{i \in [k]}, \{\theta^0_i\}_{i \in [k]})$ \tcp{to be used for the 1st update.}
\STATE \textbf{Output}: $\{\theta^0_i\}_{i \in [k]}$
\end{algorithmic}
\caption{$\cAdistr$: Distributed Learning Algorithm}
\label{alg:learning-distributed0}
\end{algorithm}

The state $s$ of the unlearning algorithm  $\cR_{\cAdistr}$ records  the $k$ shards $\{D_i\}_{i}$ and the ensemble of $k$ models $\{\theta_i\}_i$. Thus $\stateSpace = \{\cZ^*\}^k \times \Theta^k$. As an update request $u$ is received, the update function removes the data point from every shard that contains it (for deletion) or adds the new point to each shard with probability $p$ (for addition). In either case, only the models corresponding to shards that have been updated are retrained using $\cAsingle$. We formally describe $\cR_{\cAdistr}$ in Algorithm~\ref{alg:unl-distributed}.

\begin{algorithm}[t]
\SetAlgoLined 
\begin{algorithmic}
\STATE \textbf{Input}: dataset $D^{t-1}$, update $u^t=(z^t, \bullet^t)$, state $s^{t-1} = (\{D^{t-1}_i\}_{i \in [k]}, \{\theta^{t-1}_i\}_{i \in [k]})$
\IF{$\bullet^t = \mathtt{'delete'}$}
	\STATE $S = \{i \in [k]: z^t \in D^{t-1}_i\}$ \tcp{the shards $z^t$ belongs to.}
\ELSE
    \STATE $S = \{i \in [k]: \Sampler_i(\{z^t\}, p) \neq \{\}\}$ \tcp{the shards $z^t$ will be added to.}
\ENDIF

\STATE Update the shards: $D^{t}_i = \begin{cases}
	D^{t-1}_i \circ u^t & \text{if $i \in S$}\\
	D^{t-1}_i & \text{otherwise}
\end{cases}$, for every $i \in [k]$.
\STATE Update the models: $\theta^{t}_i = \begin{cases}
	\cAsingle (D^{t}_i) & \text{if $i \in S$}\\
	\theta^{t-1}_i & \text{otherwise}
\end{cases}$, for every $i \in [k]$.
\STATE Update the state: $s^t = (\{D^t_i\}_{i \in [k]}, \{\theta^t_i\}_{i \in [k]})$ \tcp{to be used for the next update.}
\STATE \textbf{Output}: $\{\theta^t_i\}_{i \in [k]}$
\end{algorithmic}
\caption{$\cR_{\cAdistr}$: Distributed Unlearning Algorithm: $t$'th round of unlearning}
\label{alg:unl-distributed}
\end{algorithm}

First, we show that if the update requester is non-adaptive, $\cR_{\cAdistr}$ is a $(0,0,0)$-unlearning algorithm:

\begin{restatable}[]{lemma}{lemNonAdaptive}
\label{lem:non-adaptive}
$\cR_{\cAdistr}$ is a non-adaptive $(0, 0, 0)$-unlearning algorithm for $\cAdistr$.
\end{restatable}

Now, by combining Lemma \ref{lem:non-adaptive} and our general Theorem \ref{thm:general-theorem}, we can show the following:
\begin{restatable}[Unlearning Guarantees]{theorem}{thmAdaptiveDeletion}
\label{thm:adaptive-deletion}
If for every round $t$, the sequence of publishing functions  $\{ \publish^{t'} \}_{t' \le t}$ is $(\epsilon,\delta)$-differentially private in the random seeds $r \sim \cP^k$ of the algorithms for $0 < \epsilon \le 1/2$ and $0 < \delta < \epsilon$, then  $\cR_{\cAdistr}$ is an $(\alpha, \beta, \gamma)$-unlearning algorithm for $\cAdistr$ where
\[
\alpha = O\left(\epsilon^2 k + k \sqrt{\delta / \epsilon}\right), \quad
\beta = \gamma = O\left(\sqrt{e^{-\epsilon^2 k} + k\sqrt{\delta / \epsilon}}\right)
\]
\end{restatable}

Next, we bound the time complexity of our algorithms:

\begin{restatable}[Run-time Guarantees]{theorem}{thmRunTime}
\label{thm:run-time}
Let $p = 1/k$. Suppose the publishing functions satisfy the differential privacy requirement of Theorem~\ref{thm:adaptive-deletion}. Let $N^t$ denote the number of times $\cR_\cAdistr$ calls $\cAsingle$ at round $t$. We have that $N^0 = k$, and for every round $t \ge 1$: 1) if the update requester is non-adaptive, for every $\xi$, with probability at least $1-\xi$, $N^t \le 1 + \sqrt{2 \log \left( 1 / \xi \right)}$. 2) if the update requester is adaptive, for every $\xi$, with probability at least $1-\xi$, $N^t \le 1 + \sqrt{2 \log \left( (n+t) / \xi \right)}$. Furthermore, for $\xi > \delta'$, with probability at least $1-\xi$, we have
\[
N^t \le 1 + \min \left\{ \sqrt{2 \log \left( 2(n+t) / (\xi - \delta') \right)}, \sqrt{2 \epsilon' + 2\log \left( 2 / (\xi - \delta') \right)} \right\}
\]
where $\epsilon' = O\left(\epsilon^2 k + k \sqrt{\delta / \epsilon}\right)$ and $\delta' =  e^{-\epsilon^2 k} + O\left( k\sqrt{ \delta / \epsilon}\right)$
\end{restatable}
The proof can be found in the appendix, but at a high level it proceeds as follows. For a deletion request, we must retrain every shard that contains the point to be deleted. For a non-adaptive deletion request, we retrain one shard in expectation and we can obtain a high probability upper bound by using a Hoeffding bound. In the adaptive case, this may no longer be true, but there are two ways to obtain upper bounds that correspond to the two bounds in our Theorem. We can provide a worst-case upper bound on the number of shards that \emph{any} of the $n$ data points belongs to, which incurs a cost of order $\sqrt{\log n}$. Alternately,  we can apply max-information bounds to reduce to the non-adaptive case, using an argument that is similar to our reduction for deletion guarantees.

\subsection{Private Aggregation}
We briefly describe how we serve prediction requests by privately aggregating the output of the ensemble of models such that the published predictions are differentially private in the random seeds $r$. 
At each round $t$, while $\cR_\cAdistr$ is waiting for the next update request $u^{t+1}$, we  receive prediction requests $x$ and serve predictions $\hat y$. For each prediction request, we privately aggregate the predictions made by the ensemble of models $\{\theta^t_i\}_i$; \cite{privatepred} show several ways to privately aggregate predictions (one simple technique is to use the exponential mechanism to approximate the majority vote). Suppose we aggregate the predictions made by the ensemble of models using $\ppredict^k_{\epsilon'}: \Theta^k \times \cX \to \mathcal{Y}$, which 
takes in an ensemble of $k$ models and a data point, aggregates predictions from the ensemble models, and outputs a label that is $\epsilon'$-differentially private in the models. If we receive $l^t$ many prediction requests $(x^t_1, \dots, x^t_{l^t})$ before our next update request $u^{t+1}$, we can write $(\hat{y}^t_1, \dots, \hat{y}^t_{l^t}) = \publish^t(\{\theta^t_i\}_i)$ where $\hat{y}^t_j = \ppredict^k_{\epsilon'}(\{\theta^t_i\}_i, x^t_j)$. 

 Theorem~\ref{thm:adaptive-deletion}, tells us that  desired unlearning parameters $(\alpha,\beta,\gamma)$ can be obtained by guaranteeing that the sequence of predictions is $(\epsilon,\delta)$ differentially private in the models (and hence $r$), for target parameters $\epsilon,\delta$. As we serve prediction requests using $\ppredict^k_{\epsilon'}$ our privacy loss will accumulate and eventually exhaust our budget of $(\epsilon,\delta)$-differential privacy. Hence we must track our accumulated privacy loss in the state of our unlearning algorithm, and when it is exhausted, fully retrain using $\cAdistr$. This resamples $r$ and hence resets our privacy budget.  
 Standard composition theorems (see \cite{DR14}) show that we exhaust our privacy budget (and need to fully retrain) every time  the number of prediction requests made since the last full retraining exceeds $\left\lfloor \frac{\epsilon^2}{8(\epsilon')^2 \ln(\frac{1}{\delta})} \right\rfloor$. We formally describe this process denoted as $\mathtt{PrivatePredictionInteraction}(\epsilon', \epsilon, \delta, k)$ in the appendix and state its unlearning guarantee in Theorem~\ref{thm:final}.

\begin{restatable}[]{theorem}{thmFinal}
\label{thm:final}
The models $\{\{\theta^t_i\}_i\}_t$ in $\mathtt{PrivatePredictionInteraction}(\epsilon', \epsilon, \delta, k)$ satisfy $(\alpha,\beta,\gamma)$-unlearning guarantee for $\cAdistr$ where $\alpha = O\left(\epsilon^2 k + k \sqrt{\delta / \epsilon}\right)$ and $\beta,\gamma = O\left(\sqrt{e^{-\epsilon^2 k} + k\sqrt{\delta / \epsilon}}\right)$, if $0 < \epsilon \le 1/2$ and $0 < \delta < \epsilon$.
\end{restatable}

\section{Evaluation of Unlearning Guarantees}
\label{sec:experiment}
In this section we demonstrate that the deletion guarantees of algorithms in the SISA framework \citep{unlearning} fail for adaptive deletion sequences.  In Section \ref{sec:labelonly} we give a clean toy construction which shows algorithms in the SISA framework fail to have nontrivial adaptive deletion guarantees even in the black-box setting when the models within each shard are not made public, only aggregations of their classification outputs. In the Appendix we experimentally evaluate a more realistic instantiation of this construction. In Section \ref{sec:fullmodelsetting} we consider the white-box setting in which the models in each shard are made public. SISA continues to have perfect deletion guarantees against \emph{non-adaptive} deletion sequences in this setting. Experimental results on CIFAR-10 \citep{cifar10}, MNIST \citep{mnist}, and Fashion-MNIST \citep{fmnist} show both the failure of SISA to satisfy adaptive deletion guarantees, and give evidence that differential privacy can mitigate this problem well beyond the setting of our theorems while achieving accuracy only modestly worse than SISA. The code for our experiments can be found at \url{https://github.com/ChrisWaites/adaptive-machine-unlearning}.

\subsection{Theory for the Label-Only Setting}
\label{sec:labelonly}
The first setting we consider directly corresponds to the setting in which our final algorithms operate: what is made public is the aggregate predictions of the ensemble of models, but not the models themselves. For non-adaptive sequences of deletions, distributed algorithms of the sort described in Section \ref{sec:distributed} have perfect deletion guarantees. We demonstrate via a simple example that these guarantees dramatically fail for adaptive deletion sequences.

Suppose we have a dataset consisting of real-valued points with binary labels $\{ (x_i, y_i) \}_{i=1}^{2n}$, $x_i \in \mathbb{R}^d$, $y_i \in \{0, 1\}$ in which there are exactly two copies of each distinct training example. Consider a simplistic classification model, resembling a lookup table, which given a point $x_i$ predicts the label $y_i$ if the model has been trained on ($x_i, y_i$) and a dummy prediction value "$\bot$" otherwise:
\[
f_{\mathcal{D}}(x_i)=
    \begin{cases}
    y_i &\text{if } (x_i, y_i) \in \mathcal{D},\\
    \bot &\text{otherwise}
    \end{cases}
\]

Consider what happens when the training algorithm randomly partitions this dataset into three pieces and trains such a model on each partition. This constructs an ensemble which, at query time, predicts the class with the majority vote. On this dataset, the ensemble will predict the labels of roughly $2 / 3$ of the training points correctly---that is, exactly those points for which the duplicates have fallen into distinct partitions, so that the ensemble gets the majority vote right. 

We construct an adaptive adversary who chooses to delete exactly those training points that the ensemble correctly classifies (which are those points for whom the duplicates have fallen into distinct shards). The result is that the model resulting from this deletion sequence will misclassify every remaining training point. Full retraining (because it would rerandomize the partition) would again lead to training accuracy of approximately $2/3$. Recalling that our deletion notion requires that the probability of any event under the unlearning scheme is not  much larger than the probability of the same event under full retraining, this demonstrates that there are algorithms in the SISA framework --- even if the models are not directly exposed --- that do not satisfy $(\alpha,\beta, \gamma)$-deletion guarantees for any nontrivial value of $\alpha$. We formalize this below:

\begin{restatable}[]{theorem}{thmSisaFailure}
\label{thm:sisa-failure}
There are learning and unlearning algorithms in the SISA framework $(\cA, \cR_\cA)$ such that for any $\alpha$, and any $\beta,\gamma < 1/4$, $\cR_{\cA}$ is not an $(\alpha,\beta,\gamma)$-unlearning algorithm for $\cA$.
\end{restatable}

A  proof of this theorem  can be found in the appendix.

\subsection{Experiments for the Full-Model Setting}
\label{sec:fullmodelsetting}

We train SISA with an ensemble of convolutional neural networks on several datasets of points with categorical labels. Given a new point at query time,  each model in the ensemble votes on the most likely label and aggregates their votes. The models are exposed publicly. This scheme has perfect non-adaptive deletion guarantees.

To construct an adaptive deletion sequence to falsify the hypothesis that the scheme has adaptive deletion guarantees, we exploit the observation that neural networks are often \textit{overconfident} in the correct label for points on which they have been trained. For each training point, we guess that it falls into the shard corresponding to the model that has the highest confidence for the correct label. We then delete points for which we guess that they fall into the first $k/2$ of the shards, and do not delete any others. After deleting the targeted points, we compute a test statistic: the indicator of whether the average accuracy of the models from the targeted shards is lower than the average accuracy of the models from the non-targeted shards. Under full retraining, by the symmetry of the random partition, the expectation of this test statistic is 0.5. Thus under the null hypothesis that the deletion algorithm satisfies perfect deletion guarantees, the test statistic also has expectation 0.5. Therefore, to the extent that the expectation of the indicator differs from 0.5, we falsify the null hypothesis that SISA has adaptive data deletion guarantees, and larger deviations from 0.5 falsify weaker deletion guarantees. 

We run this experiment on three datasets (CIFAR-10, MNIST, and Fashion-MNIST), and plot the results in Figure \ref{fig:exp-figs}. We then repeat the experiment by adding various amounts of noise to the gradients in the model training process to guarantee finite levels of differential privacy (though much weaker privacy guarantees than would be needed to invoke our theorems). We observe that on each dataset, modest amounts of noise are sufficient to break our attack (i.e. 95\% confidence intervals for the expectation of our indicator include $0.5$, and hence fail to falsify the null hypothesis) while still approaching the accuracy of our models trained without differential privacy. This is also plotted in Figure \ref{fig:exp-figs}. This gives evidence that differential privacy can improve deletion guarantees in the presence of adaptivity even in regimes beyond which our theory gives nontrivial guarantees.

\begin{figure}
\begin{tabular}{ccc}
  \includegraphics[width=0.31\textwidth]{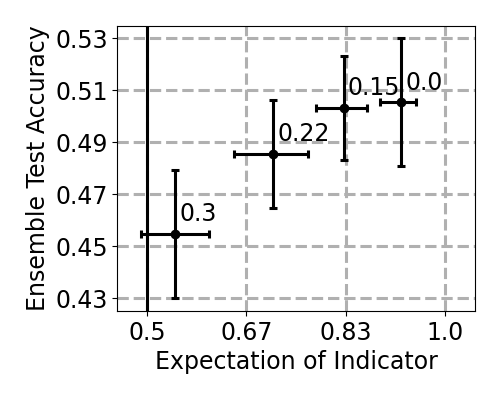} & \includegraphics[width=0.31\textwidth]{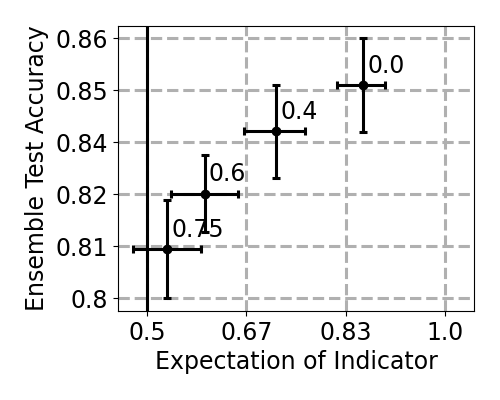} & \includegraphics[width=0.31\textwidth]{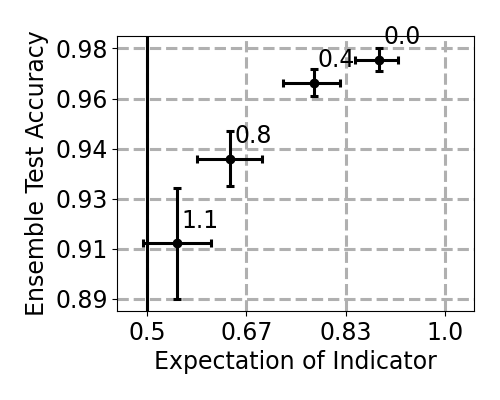} \\
  (a) $\text{CIFAR-10}_{k=6}$ & (b) $\text{Fashion-MNIST}_{k=6}$ & (c) $\text{MNIST}_{k=6}$ \\
  \includegraphics[width=0.31\textwidth]{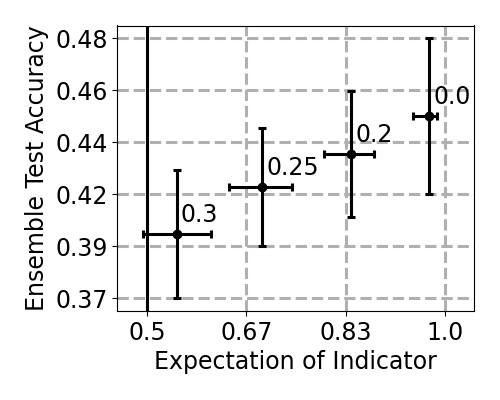} & \includegraphics[width=0.31\textwidth]{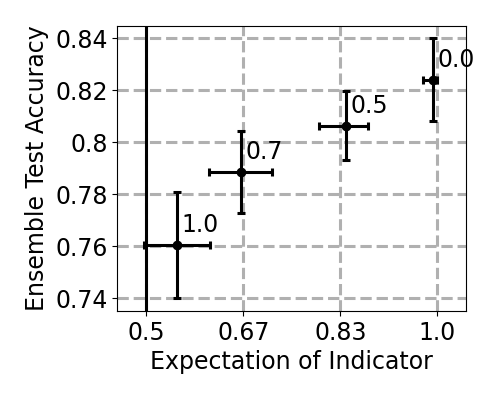} & \includegraphics[width=0.31\textwidth]{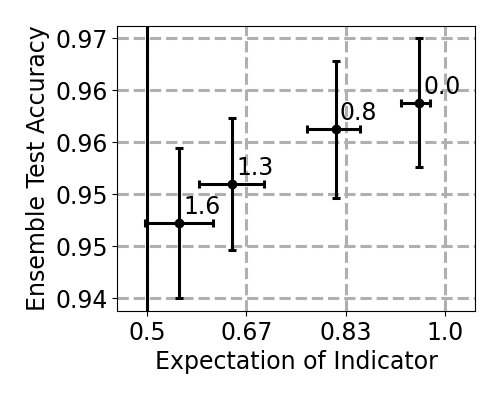} \\
  (d) $\text{CIFAR-10}_{k=2}$ & (e) $\text{Fashion-MNIST}_{k=2}$ & (f) $\text{MNIST}_{k=2}$ \\
\end{tabular}
\caption{The top row and bottom row show experiments with $k = 6$ and $k = 2$ shards respectively. The 3 columns report on 3 datasets. The $x$ axis denotes estimated expectation of our test statistic (the null hypothesis is expectation $0.5$). The $y$ axis denotes the accuracy of the ensemble after deletion. Each point is annotated with the noise multiplier used in DP-SGD, the standard deviation of Gaussian noise applied to gradients during training. A label of $0.0$ for a point represents the baseline case of no noise (original SISA algorithm). Points are affixed with 95\% confidence intervals along both axes (over the randomness of repeating the training/deletion experiment). Horizontal confidence intervals that overlap the line denoting expectation $0.5$ fail to reject the null hypothesis that the algorithm has adaptive data deletion guarantees at $p \leq 0.05$. We get to this point with a level of noise addition that results in only a modest degradation in ensemble performance compared to SISA.}
\label{fig:exp-figs}
\end{figure}

Full experimental details can be found in the appendix.

\section{Conclusion and Discussion}
We identify an important blindspot in the data deletion literature (the tenuous implicit assumption that deletion requests are independent of previously released models), and provide a very general methodology  to reduce adaptive deletion guarantees to oblivious deletion guarantees. Through this reduction we get the first model and training algorithm agnostic  methodology that allows for deletion of arbitrary sequences of adaptively chosen points while giving rigorous guarantees. The constants that our theorems inherit from the max information bounds of \cite{bounded-max-info} are such that in most realistic settings they will not give useful parameters. But we hope that these constants will be improved in future work, and we give empirical evidence that differential privacy mitigates adaptive deletion ``attacks'' at very practical levels, beyond the promises of our theoretical results. We note that like for differential privacy, the $(\alpha,\beta,\gamma)$-deletion guarantees we give in this paper are \emph{parameterized}, and are not meaningful absent a specification of those parameters. There is a risk with such technologies that they will be used with large values of the parameters that give only very weak guarantees, but will be described publicly in a way that glosses over this issue. We therefore recommend that if adopted in deployed products, deletion guarantees always be discussed in public in a way that is precise about what they promise, including the relevant parameter settings. 

\bibliographystyle{plainnat}
\bibliography{refs}

\newpage

\appendix

\section{Proof of Theorem~\ref{thm:general-theorem}}
We first state the following Lemma which we will use to prove Theorem~\ref{thm:general-theorem}. 

\begin{lemma} [\citep{bounded-max-info}]\label{lem:maxinf}
Let $M: \cX^m \rightarrow \cO$ be an $(\epsilon, \delta)$-differentially private algorithm for $0 < \epsilon \le 1/2$ and $0 < \delta < \epsilon$. Then,
\[
\Pr_{(x,m') \sim (X, M(X))} \left[ \log \left( \frac{\Pr \left[ X = x, M(X) = m' \right]}{\Pr \left[ X = x \right] \Pr \left[ M(X) = m' \right]} \right) \ge k \right] \le \beta
\]
where the probability is taken with respect to the joint distribution of $X$ and $M(X)$, and
\[
k = O\left(\epsilon^2 m + m \sqrt{\frac{\delta}{\epsilon}}\right), \quad \beta = e^{-\epsilon^2m} + O\left(m\sqrt{\frac{\delta}{\epsilon}}\right)
\]
\end{lemma}

\thmGeneralTheorem*
\begin{proof}
Fix a data set $D$ and an update requester $\updreq$. Fix any unlearning step $t \ge 1$. Note that the sequence of updates up to round $t$, i.e. $u^{\le t} = (u^1, \ldots, u^t)$, can be seen as a post-processing of the sequence of published objects up to round $t-1$, i.e. $\psi^{\le t-1} = (\psi^0, \ldots, \psi^{t-1})$, where the post-processing function is defined by $\updreq$ (see Definition~\ref{def:updreq}). But we know that  $\{ \publish^{t'} \}_{t' \le t-1}$ that generates $\psi^{\le t-1}$ is $(\epsilon,\delta)$-differentially private in $r$. Hence, given that post-processing preserves differential privacy (Lemma~\ref{lem:postprocessing}), we have that $u^{\le t}$ is also $(\epsilon,\delta)$-differentially private in $r$.
Consequently, we can apply the fact that DP implies bounded max-information (Lemma~\ref{lem:maxinf}) to get that
\begin{equation}\label{eq:max-bound}
\Pr_{(r,u^{\le t})} \left[  \log \frac{\Pr \left[ r \vert u^{\le t} \right]}{\Pr \left[ r \right] }  \ge  \epsilon' \right] = \Pr_{(r,u^{\le t})} \left[  \log \frac{\Pr \left[ r, u^{\le t} \right]}{\Pr \left[ r \right] \Pr \left[ u^{\le t} \right]}  \ge  \epsilon' \right]  \le \delta'
\end{equation}
where the probability is taken with respect to the joint distribution of $(r,u^{\le t})$, and that
\[
\epsilon' \triangleq O\left(\epsilon^2 m + m \sqrt{\frac{\delta}{\epsilon}}\right), \quad \delta' \triangleq  e^{-\epsilon^2 m} + O\left( m\sqrt{\frac{\delta}{\epsilon}}\right)
\]
Now define the ``Good" event for the update sequence $u^{\le t}$:
\[
G = \left\{ u^{\le t} : \Pr_{r \vert u^{\le t}} \left[  \log \frac{\Pr \left[ r \vert u^{\le t} \right]}{\Pr \left[ r \right] }  \ge  \epsilon'\right] \le \sqrt{\delta'} \right\}
\]
We have that
\begin{align*}
\Pr_{u^{\le t}} \left[ u^{\le t} \notin G \right] &= \Pr_{u^{\le t}} \left[ \Pr_{r \vert u^{\le t}} \left[ \log \frac{\Pr \left[ r \middle\vert u^{\le t} \right]}{\Pr \left[ r \right] }  \ge  \epsilon' \right] > \sqrt{\delta'} \right] \\
&\le \frac{\E_{u^{\le t}} \left[ \Pr_{r \vert u^{\le t}} \left[  \log \frac{\Pr \left[ r \vert u^{\le t} \right]}{\Pr \left[ r \right] }  \ge  \epsilon' \right] \right]}{\sqrt{\delta'}} \\
&= \frac{\Pr_{(r,u^{\le t})} \left[  \log \frac{\Pr \left[ r \vert u^{\le t} \right]}{\Pr \left[ r \right] }  \ge  \epsilon' \right]}{\sqrt{\delta'}} \\
&\le \sqrt{\delta'}
\end{align*}
where the first inequality is an application of Markov's inequality, and the last one follows from Equation~\eqref{eq:max-bound}. Therefore, if we condition on $\{ u^{\le t} \in G \}$ which happens with probability at least $1 - \sqrt{\delta'}$, we have the following guarantee.
\[
\Pr_{r \vert u^{\le t}} \left[  \log \frac{\Pr \left[ r \vert u^{\le t} \right]}{\Pr \left[ r \right] }  \ge  \epsilon' \right] \le \sqrt{\delta'}
\]
which in turn implies, with probability $1 - \sqrt{\delta'}$ over the draw of $u^{\le t}$, that for every event $F$ in the space of random seeds ($r$),
\begin{equation}\label{eq:blah}
\Pr \left[ r \in F \, \vert \, u^{\le t}\right] \le e^{\epsilon'} \Pr \left[ r \in F \right] + \sqrt{\delta'}
\end{equation}
Now we condition on $\{ u^{\le t} \in G \}$. Fix any event $E \subseteq \Theta^*$ in the space of models, and let $F = \{ r: \cR_\cA (D^{t-1}, u^{t} , s^{t-1}) \in E\}$ be the event that the output models of the unlearning algorithm on round $t$ belongs to $E$, recalling that $s^{t-1} = g^{t-1} (D^0, u^{\le t-1}, r)$. Substituting $F$ in Equation~\eqref{eq:blah}, we get that
\begin{equation}\label{eq:1}
\Pr \left[ \cR_\cA (D^{t-1}, u^{t} , s^{t-1}) \in E \, \vert \, u^{\le t}\right] \le e^{\epsilon'} \Pr \left[ \cR_\cA (D^{t-1}, u^{t} , s^{t-1}) \in E \right] + \sqrt{\delta'}
\end{equation}
Note that because on the right hand side we do not condition the probability on the update sequence, we are taking the probability over the distribution of output models of round $t$ for a nonadaptively chosen update sequence. Therefore by the unlearning guarantees for nonadaptive update requesters, we have that with probability at least $1-\gamma$ over the draw of $u^{\le t}$,
\begin{equation}\label{eq:2}
\Pr \left[ \cR_\cA (D^{t-1}, u^{t} , s^{t-1}) \in E \right] \le e^\alpha \Pr \left[ \cA (D^t) \in E \right] + \beta
\end{equation}
Now we can combine Equations~\eqref{eq:1} and \eqref{eq:2} to conclude that, with probability $1-\gamma - \sqrt{\delta'}$ over $u^{\le t}$,
\[
\Pr \left[ \cR_\cA (D^{t-1}, u^{t} , s^{t-1}) \in E \, \vert \, u^{\le t}\right] \le e^{\alpha + \epsilon'} \Pr \left[ \cA (D^t) \in E \right] + \beta e^{\epsilon'} + \sqrt{\delta'}
\]
completing the proof.
\end{proof}

\section{Missing Details from Section \ref{sec:distributed}}

\begin{lemma}
\label{lem:shard-dist}
Consider the distributed learning and unlearning algorithms $\cAdistr$ and $\cR_\cAdistr$. If the update requester is non-adaptive, for every $t$: for every shard $i$, we have $D_i^t$  is an independent draw from the distribution of $\Sampler (D^t, p)$.
\end{lemma}
\begin{proof}
We prove this via induction. It's easy to see that this holds true at round $t=0$ because we explicitly set $D^0_i = \Sampler(D^0, p)$. 
Now, suppose that $D^{\tau-1}_i$ is an independent draw from the distribution of $\Sampler(D^{\tau-1}, p)$ for some $\tau \ge 1$. If the update request $u^\tau = (z^\tau, \mathtt{'delete'})$ is a deletion request, then it's easy to see that simply deleting the point $z^\tau$ from every shard that contains it will maintain that each element is chosen to be in the shard with probability $p$. And $D^{\tau}_i | u^{\tau-1}$ and $D^{\tau}_i | u^{\tau}$ must be identically distributed because the update request $u^\tau$ is non-adaptive and has been fixed prior to the interaction --- and hence is statistically independent of $D^{\tau-1}$. More formally, we have that for any $z \in D^\tau$,
\begin{align*}
    \Pr[z \in D^{\tau}_i ] = \Pr[z \in D^{\tau}_i | u^{\le \tau}] = \Pr[z \in D^{\tau}_i | u^{\le \tau-1}] = \Pr[z \in D^{\tau-1}_i | u^{\le \tau-1}] = p.
\end{align*}

The same argument applies for the addition request where $\cR_\cAdistr$ adds the element requested to be added with probability $p$. More formally, we have $\Pr[z \in D^{\tau}_i] = p$ for any $z \in D^{\tau-1}$ and $\Pr[z^\tau \in D^\tau_i] = p$ by construction.
\end{proof}

\lemNonAdaptive*
\begin{proof}
Fix any arbitrary round $t \in [T]$. For a non-adaptive $\updreq$, we can think of the update sequence $u^{\leq t}$ as fixed prior to the start of the interaction between the learning procedure and the $\updreq$. Now, in order to show $(0,0,0)$-deletion guarantee of the unlearning algorithm, we need to show that for any $E \subseteq \Theta^*$,
\[
    \Pr\left[\cR_{\cAdistr}(D^{t-1}, u^t, s^{t-1}) \in E | u^{\le t}\right] = \Pr\left[\cAdistr(D^t) \in E\right]. 
\]
Note that it is equivalent to show that for any $i \in [k]$ and $E \subseteq \Theta$, we have
\[
\Pr\left[\theta^t_i \in E| u^{\le t} \right] = \Pr\left[\cAsingle_i(\Sampler_i(D^t, p)) \in E\right]
\]
because $\Sampler_i$ and $\cAsingle_i$ behave independently across $i \in [k]$ in both $\cR_{\cAdistr}$ and $\cAdistr$. Hence, from here on, we focus on some fixed $i \in [k]$.

Now, we argue that it is sufficient to show that the distribution over $D^{t}_i$ conditional on $u^{\le t}$ that is being kept in the state $s^t$ of the unlearning algorithm is exactly the same as that of $\Sampler_i(D^t, p)$, which we have already proved in Lemma~\ref{lem:shard-dist}. Using the fact that update sequence is non-adaptive with respect to the algorithm's randomness, we have for any realization path for shard $i$ until round $t$ (i.e. how the initial shard $D^0_i$ was formed and whether each addition request until round $t$ was actually added to shard $i$ or not)
\begin{align*}
    \Pr[\theta^t_i \in E | u^{\le t}] &= \Pr[\theta^{t'}_i  \in E| u^{\le t}]\\
    &= \Pr[\theta^{t'}_i \in E| u^{\le t'}]\\
    &= \Pr[\cAsingle_i(D_i^{t'}) \in E| u^{\le t'}]\\
    &= \Pr[\cAsingle_i(D_i^{t}) \in E| u^{\le t}]
\end{align*}
where $t' = \min \{\tau \le t: D_i^{\tau} = D^t_i\}$ is the time at which we last trained the model for shard $i$ in the unlearning algorithm.
\end{proof}

\thmAdaptiveDeletion*
\begin{proof}
Lemma \ref{lem:non-adaptive} provides that $\cR_{\cAdistr}$ is a $(0,0,0)$-unlearning algorithm for $\cAdistr$ against any nonadaptive update requester.

Note that because the randomness used in each shard $i \in [k]$ is always independent and there is a symmetry across these shards in both $\cAdistr$ and $\cR_{\cAdistr}^{\text{iter}}$, we can imagine drawing all the randomness required for each shard throughout the interaction prior to the interaction $r \sim \cP^k$ such that each shard $i \in [k]$ relies $r_i$ on as the source of its randomness. 

Now, note that the state kept by $\cR_\cAdistr$ consists of the shards $\{D^{t-1}_i\}_i$ and the models trained via $\cAsingle$ on those shards $\{\theta^{t-1}_i\}_i$. Hence, at any round $t$, given access to initial dataset $D^0$, previous update requests $u^{\le t-1}$, and the randomness that has been drawn prior to the interaction $r$, we can deterministically determine the state $s^{t-1} = (\{D^{t-1}_i\}_i, \{\theta^{t-1}_i\}_i)$, meaning there exists some deterministic mapping $g^{t-1}$ such that $s^{t-1} = g^{t-1}(D^0, u^{\le t-1}, r)$. 

Therefore, we can combine the $(0,0,0)$-deletion guarantee promised by Lemma \ref{lem:non-adaptive} with Theorem \ref{thm:general-theorem} to conclude that $\cR_\cAdistr$ must be $(\alpha, \beta, \gamma)$-unlearning algorithm for $\cAdistr$.
\end{proof}

\thmRunTime*
\begin{proof}
Throughout we use $Bin(k,p)$ to denote a binomial random variable with parameters $k$ (number of trials) and $p$ (success probability). First we state the following fact:
\begin{fact}[Binomial Tail Bound]\label{fact:binom}
Let $X \sim Bin(k,p)$ and let $\mu := kp$. We have that for every $\eta \ge 0$,
\[
\Pr \left[ X \ge (1+\eta)\mu \right] \le e^{-\frac{\eta^2 \mu}{2 + \eta}}
\]
which in turn implies, for every $\delta$, with probability at least $1-\delta$,
\[
X \le \left( 1 + \frac{\sqrt{\log^2 \left( 1 / \delta \right) + 8 \mu \log \left( 1 / \delta \right)} -\log \left( 1 / \delta \right)}{2 \mu} \right) \mu \le \mu + \sqrt{2 \mu \log \left( 1 / \delta \right)}
\]
\end{fact}
Fix any round $t \ge 1$ of the update, and let $\mu = kp$ throughout. Suppose the update requester is non-adaptive. If the update of round $t$ is an addition, then $N^t \sim Bin(k,p)$ by construction. If the update of round $t$ is a deletion: $u^t = (z^t, \mathtt{'delete'})$, then 
\[
N^t = \sum_{i=1}^k \ind \left[ z^t \in D_i^{t-1} \right]
\]
But the update requester being non-adaptive (implying $z^t$ is independent of the randomness of the algorithms), together with Lemma~\ref{lem:shard-dist}, imply that $N^t$ is a sum of \emph{independent} Bernoulli random variables with parameter $p$; hence, $N^t \sim Bin(k,p)$. Therefore, if the update requester is non-adaptive, we can apply Fact~\ref{fact:binom} to conclude that for every $\xi$, with probability at least $1-\xi$, we have
\[
N^t \le \mu + \sqrt{2 \mu \log \left( 1 / \xi \right)}
\]
which proves the first part of the theorem for the choice of $p=1/k$. Now suppose the update requester is adaptive. If the update of round $t$ is an addition, then $N^t \sim Bin(k,p)$ by construction, and therefore using Fact~\ref{fact:binom}, with probability at least $1-\xi$, we have
$
N^t \le \mu + \sqrt{2 \mu \log \left( 1 / \xi \right)}
$.
Now suppose the update is a deletion: $u^t = (z^t, \mathtt{'delete'})$. We have in this case that
\[
N^t = \sum_{i=1}^k \ind \left[ z^t \in D_i^{t-1} \right]
\]
First note that we have the following upper bound
\begin{equation}\label{eq:worstcase1}
    N^t \le \sup_{z \in D^{t-1}} \sum_{i=1}^k \ind \left[ z \in D_i^{t-1} \right] \le \sup_{z \in D^0 \cup \{ z^1, \ldots, z^{t-1} \}} \sum_{i=1}^k \ind \left[ z \in D_i^{t-1} \right]
\end{equation}
where $\{ z^1, \ldots, z^{t-1} \}$ are the data points that have been requested to be added or deleted by the update requester in the previous rounds. Here, in the worst case (to get upper bounds), we are assuming that all previous $t-1$ updates are addition requests. Note that for every $z \in D^0 \cup \{ z^1, \ldots, z^{t-1} \}$, the number of shards that contain $z$ is an independent draw from a $Bin(k,p)$ distribution, by construction. We therefore have that
\begin{equation}\label{eq:worstcase2}
\sup_{z \in D^0 \cup \{ z^1, \ldots, z^{t-1} \}} \sum_{i=1}^k \ind \left[ z \in D_i^{t-1} \right] \overset{d}{=} \sup_{1 \le j \le n+t-1} X_j
\end{equation}
where the equality is in distribution, and $X_j \sim Bin(k,p)$. Now, combining Equations~\eqref{eq:worstcase1} and \eqref{eq:worstcase2}, and using Fact~\ref{fact:binom}, we get that for every $\eta \ge 0$,
\[
\Pr \left[ N^t \ge \left( 1 + \eta \right) \mu \right] \le \sum_{j = 1}^{n+t-1} \Pr \left[ X_j \ge \left( 1 + \eta \right) \mu \right] \le (n+t) e^{-\frac{\eta^2 \mu}{2 + \eta}} 
\]
which implies, for every $\xi \ge 0$, with probability at least $1-\xi$,
\begin{equation}\label{eq:bound1}
N^t \le \mu + \sqrt{2 \mu \log \left( (n+t) / \xi \right)}.
\end{equation}
We will prove another upper bound using the max-information bound. Recall that our distributed algorithms can be seen as drawing all the randomness $r \sim \cP^k$ upfront for some distribution $\cP$ (one draw from $\cP$ per shard). Since the update sequence $u^{\le t}$ (which is a post processing of the published objects) is guaranteed to be $(\epsilon, \delta)$-differentially private in $r$, we get using the max-information bound that, for every $\eta \ge 0$,
\begin{equation}\label{eq:maxinfos}
\Pr \left[ N^t \ge \left( 1 + \eta \right) \mu \right] \le e^{\epsilon'} \Pr_{(r \otimes u^{\le t})} \left[ N^t \ge \left( 1 + \eta \right) \mu \right] + \delta'
\end{equation}
where on the left hand side the probability is taken with respect to the joint distribution of $r$ and $u^{\le t}$, and on the right hand side $(r \otimes u^{\le t})$ means $r$ and $u^{\le t}$ are drawn independently from their corresponding marginal distributions. But when $r$ and $u^{\le t}$ are drawn independently (i.e., the update requester is non-adaptive), $N^t \sim Bin(k,p)$ as we have shown in the first part of this theorem.
\begin{equation}\label{eq:nonadaptives}
\Pr_{(r \otimes u^t)} \left[ N^t \ge \left( 1 + \eta \right) \mu \right] = \Pr \left[ Bin (k,p) \ge \left( 1 + \eta \right) \mu \right] \le e^{-\frac{\eta^2 \mu}{2 + \eta}}
\end{equation}
Therefore, combining Equations~\eqref{eq:maxinfos} and \eqref{eq:nonadaptives}, we get that
\[
\Pr \left[ N^t \ge \left( 1 + \eta \right) \mu \right] \le e^{\epsilon'-\frac{\eta^2 \mu}{2 + \eta}} + \delta'
\]
which in turn implies, for every $\xi > \delta'$, with probability at least $1-\xi$,
\begin{equation}\label{eq:bound2}
N^t \le \mu + \sqrt{2 \mu \left( \epsilon' + \log \left( 1 / (\xi - \delta') \right) \right)}
\end{equation}
Combining the bounds of Equations~\eqref{eq:bound1} and \eqref{eq:bound2}, we get that for every $\xi > \delta'$, with probability $1-\xi$,
\[
N^t \le \mu + \min \left\{ \sqrt{2 \mu \log \left( 2(n+t) / (\xi - \delta') \right)}, \sqrt{2 \mu \left( \epsilon' + 2\log \left( 2 / (\xi - \delta') \right) \right)} \right\}
\]
which completes the proof by the choice of $p=1/k$ ($\mu = k p = 1$).
\end{proof}

\begin{algorithm}[h]
\SetAlgoLined
\SetNoFillComment
\begin{algorithmic}
\STATE $l = 0$
\FOR{$t=1, \dots, T$}
    \IF{$l > \left\lfloor \frac{\epsilon^2}{8(\epsilon')^2 \ln(\frac{1}{\delta})} \right\rfloor$ \tcp{``Restart'' $\cR_\cAdistr$ when privacy budget is exhausted}}
        \STATE $D^{t}_i = \Sampler_i(D^{t}, p)$ and $\theta^{t}_i = \cAsingle_i(D^t_i)$ for each $i \in [k]$
        \STATE Update $s^t=(\{D^t_i\}_i, \{\theta^t_i\}_i)$ 
        \STATE $l = 0$ 
    \ELSE
        \STATE $\{\theta^t_i\}_i = \cR_{\cAdistr}(D^{t-1}, u^t, s^{t-1})$ 
    \ENDIF
    \WHILE{there is a prediction request for some $x$}
        \STATE Publish $\hat{y} = \ppredict^k_{\epsilon'}(\{\theta^t_i\}_i, x)$
        \STATE $l = l+ 1$
    \ENDWHILE
\ENDFOR
\end{algorithmic}
\caption{$\mathtt{PrivatePredictionInteraction}(\epsilon', \epsilon, \delta, k)$
}
\label{alg:aggregation}
\end{algorithm}

\begin{lemma}
\label{lem:num-pred-requests}
Assume $\epsilon < 1$ and $\delta > 0$. Then, $(\hat{y}_1, \dots, \hat{y}_l)$ is $(\epsilon, \delta)$-differentially private in $\{\theta_i\}_i$ where $\hat{y}_j = \ppredict^k_{\epsilon'}(\{\theta_i\}, x_j)$ and
\[
    l=\left\lfloor \frac{\epsilon^2}{8(\epsilon')^2 \ln(\frac{1}{\delta})} \right\rfloor.
\]
\end{lemma}
\begin{proof}
This claim holds immediately by the $(\epsilon', 0)$-differential privacy of $\ppredict^k_{\epsilon'}$ and the advanced composition theorem. See Corollary 3.21 in \cite{DR14} for details.
\end{proof}

\thmFinal*
\begin{proof}
Suppose full retraining occurs in rounds $(t_1, t_2, \dots, t_G)$ where we always have $t_1 = 0$ and $l > \left\lfloor \frac{\epsilon^2}{8(\epsilon')^2 \ln(\frac{1}{\delta})} \right\rfloor$ at round $t_g$ for any $g> 1$. 

At any round $t_g$ when full retraining occurs, we can imagine restarting $\cR_\cAdistr$ by resetting the internal round as $t=0$ and drawing fresh randomness $r \sim \cP^k$, which determines the new initial state $s^0$. Therefore, for any $g \in [G-1]$ and $t_g \le t < t_{g+1}$, we must have that $\{\publish^{t'}\}_{t_g \le t' \le t}$ are $(\epsilon, \delta)$-differentially private in the randomness $r$ drawn in round $t_g$. Then, we can appeal to Theorem~\ref{thm:adaptive-deletion} to conclude that for any $g \in [G-1]$ and $t_g \le t < t_{g+1}$, we have 
\begin{align*}
    \forall E \subseteq \Theta^*: \quad 
    \Pr \left[ \{\theta^t_i\}_i \in E \, \middle\vert \, (u_{t_g}, \dots, u_t)\right]
    \le
    e^{\alpha} \cdot \Pr \left[ \cA \left( D^t \right) \in E  \right] + \beta.
\end{align*}
Because we are redrawing fresh randomness $r \sim \cP^k$ at $t_g$, we can combine combine all the previous unlearning guarantees in the previous $(t_{g'-1},t_{g'})$ for $g' < g$ to conclude that at any round $t \in [T]$
\begin{align*}
    \forall E \subseteq \Theta^*: \quad 
    \Pr \left[ \{\theta^t_i\}_i \in E \, \middle\vert \, u^{\le t}\right]
    \le
    e^{\alpha} \cdot \Pr \left[ \cA \left( D^t \right) \in E  \right] + \beta.
\end{align*}

\end{proof}

\section{Details From Section \ref{sec:experiment}}

\subsection{Proof of Theorem \ref{thm:sisa-failure}}
\thmSisaFailure*

\begin{proof}

\begin{algorithm}[t]
\SetAlgoLined 
\begin{algorithmic}
\STATE \textbf{Input}: dataset $D \equiv D^0$ of size $n$
\STATE Draw the shards: $D^0_{i \in [k]} = \mathtt{RandomAssignPartition}(D^0, k)$.
\STATE Train the models: $\theta^0_{i \in [k]} = \cAsingle (D^0_i)$, for every $i \in [k]$.
\STATE Save the state: $s^0 = (\{D^0_i\}_{i \in [k]}, \{\theta^0_i\}_{i \in [k]})$
\STATE \textbf{Output}: $\{\theta^0_i\}_{i \in [k]}$
\end{algorithmic}
\caption{$\mathcal{A}^{\text{SISA}}$: Learning Algorithm for SISA}
\label{alg:learning-sisa}
\end{algorithm}

\begin{algorithm}[t]
\SetAlgoLined 
\begin{algorithmic}
\STATE \textbf{Input}: dataset $D^{t-1}$, update $u^t=(z^t, \bullet^t)$, state $s^{t-1} = (\{D^{t-1}_i\}_{i \in [k]}, \{\theta^{t-1}_i\}_{i \in [k]})$
\IF{$\bullet^t = \mathtt{'delete'}$}
	\STATE $i =  j \in [k] \text{, where } z^t \in D^{t-1}_j$ 
\ELSE
    \STATE $i = \mathtt{randint}(1, 2, \ldots, k)$
\ENDIF
\STATE Update the shards: $D^{t}_i = \begin{cases}
	D^{t-1}_j \circ u^t & \text{if $i = j$}\\
	D^{t-1}_j & \text{otherwise}
\end{cases}$, for every $j \in [k]$.
\STATE Update the models: $\theta^{t}_j = \begin{cases}
	\cAsingle (D^{t}_j) & \text{if $i = j$}\\
	\theta^{t-1}_j & \text{otherwise}
\end{cases}$, for every $i \in [k]$.
\STATE Update the state: $s^t = (\{D^t_i\}_{i \in [k]}, \{\theta^t_i\}_{j \in [k]})$
\STATE \textbf{Output}: $\{\theta^t_i\}_{i \in [k]}$
\end{algorithmic}
\caption{$\mathcal{R}_{\mathcal{A}^{\text{SISA}}}$: Unlearning Algorithm for SISA: $t$'th round of unlearning}
\label{alg:unlearning-sisa}
\end{algorithm}

Define $\cA^{\text{SISA}}$ and $\cR_{\cA^{\text{SISA}}}$ as Algorithms \ref{alg:learning-sisa} and \ref{alg:unlearning-sisa} respectively instantiated with the ``lookup table'' model  $\cAsingle(D) = D$ and ``lookup table'' prediction rule $f_\theta$. In Algorithm \ref{alg:learning-sisa}, $\mathtt{RandomAssignPartition}(D,k)$ assigns every $(x,y) \in D$ to one of the $k$ partitions uniformly at random. The prediction rule, given parameter $\theta = D$ and query point $x$, outputs $y$ if $(x, y) \in \theta$ and $\bot$ otherwise:
\[
    f_\theta(x) = \begin{cases}
    y &\text{if } (x, y) \in \theta,\\
    \bot &\text{otherwise}.
    \end{cases}
\]

We wish to show that there exists a dataset $D^0$ and adaptive update requester $\mathtt{UpdReq}$ such that for some update step $t \geq 1$, with probability at least $1-\gamma$ over the draw of the update sequence $u^{\le t} = (u^1, \ldots, u^t)$ from $\updreq$, $\exists E \subseteq \Theta^*: \quad \Pr \left[ \cR_\cA \left(D^{t-1}, u^{t}, s^{t-1} \right) \in E \, \middle\vert \, u^{\le t}\right] > e^{\alpha} \cdot \Pr \left[ \cA \left( D^t \right) \in E  \right] + \beta$. We prove this with the following example, instantiated for $k = 3$. 

Consider dataset $D^0$ consisting of training examples $\{ (x_i, y_i) \}_{i \in [2n]}$, $n \in \mathbb{Z}^+$ such that $D^0$ contains 2 copies each of $n$ distinct feature vectors $x$. Both copies of each distinct feature vector $x$ are paired with the same (arbitrary) label  $y$. 

Further, given ensemble model parameters $\{ \theta_i \}_{i \in [k]} = \cAdistr(D)$, let the ensemble output the mode of the predictions made by the underlying models:
\[
    \hat{y_i} = \mathtt{Mode} \left( \left\{ f_{\theta_j}(x_i) \right\}_{j \in [k]} \right).
\]

Let $\psi^0$, the published object after initial training, be the ensemble's predictions for each training point: $\psi^0 = f_{\text{publish}}^0 = ( \hat{y}_1^0, \hat{y}_2^0, \dots, \hat{y}_{2n}^0 )$. 

Given these predictions, let $I = \{ i_1, i_2, \ldots, i_t \} \subseteq [2n]$ be the indices for the points which were classified correctly. That is, $\forall i \in [2n]: i \in I$ if $y_i = \hat{y}_i$. Given $\psi^0$, let $\updreq$ be a function which outputs the deletion sequence $(u^1, u^2, \ldots, u^t)$ where each update request is responsible for deleting one of the correctly predicted points: $\forall j \in [t]: u^j = ((x_{i_j}, y_{i_j}), \mathtt{'delete'})$.

Recall that our model is parameterized by a set of model parameters $\{ \theta_i \}_{i \in [k]}$ and each $\theta_i$ is the dataset that shard is trained on. We now  define the event $E$ of interest: the set of all models such that the ensemble attains zero accuracy on the remaining points $D^t = D^0 \circ (u^1, u^2, \ldots, u^t)$, which happens if and only if all identical points (both copies of the same point) fall into the same shard.
\[
  E^t = \left\{ \{\theta_i\}_{i \in [k]} \text{ where } | \left\{ \theta_{i}: (x, y) \in \theta_i \right\} | = 1 \text{ for all } (x, y) \in D^t\right\}
\]
To make our final assertion, first note that $\Pr[\cR_{\cA^{\text{SISA}}}(D^{t-1}, u^t, s^{t - 1}) \in E | u^{\leq t}] = 1$ as $\updreq$ has requested all the correctly classified points to be deleted. We therefore need to show that 
\[
\Pr \left[ \cR_{\cA^{\text{SISA}}} \left(D^{t-1}, u^{t}, s^{t-1} \right) \in E \, \middle\vert \, u^{\le t}\right] = 1 > e^{\alpha} \cdot \Pr \left[ \cA^{\text{SISA}} \left( D^t \right) \in E \right] + \beta
\]
equivalently, $\frac{1 - \beta}{e^\alpha} > \Pr \left[ \cA^{\text{SISA}} \left( D^t \right) \in E \right]$ with probability $1-\gamma$ over the randomness of the update sequence (which in this case is simply the randomness of the initial partition). 

Note that $t$,  the number of copies of points that were initially classified correctly is distributed as $\mathtt{Binomial}(n, \frac{2}{3})$ because for each pair of identical $(x,y) \in D^0$, the probability that they fall in different shards initially is exactly $2/3$. Also, note that for any fixed $t \le n -1$, 
\begin{align*}
    \Pr[\cAdistr(D^t) \in E] = \frac{1}{3^{n-t}}.
\end{align*}

Using the tail bound for the Binomial distribution (Fact \ref{fact:binom}), we have that with probability $1-\gamma$,
\[
t \le \frac{2n}{3} + \sqrt{\frac{4n}{3}\log\left(\frac{1}{\gamma}\right)}.
\]

When $n \ge 13 \log(1/\gamma)$, we have $\frac{2n}{3} + \sqrt{\frac{4n}{3}\log\left(\frac{1}{\gamma}\right)} \le 0.99n$. Hence, for sufficiently large $n$, we can conclude that with probability $1-\gamma$, 
\[
    \Pr[\cAdistr(D^t) \in E] \le \frac{1}{3^{0.01n}}.
\]

Finally, for any $c = \frac{1 - \beta}{e^\alpha} > 0$, there exists a $D^0$ such that $c > \Pr \left[ \cA^{\text{SISA}} \left( D^t \right) \in E \right]$ with probability $1-\gamma$ because we can choose a sufficiently large $n$ such that $n \ge 13 \log(1/\gamma)$ and $\frac{1}{3^{0.01n}} \le c$, i.e., we can choose:
\[
n \ge \max \left\{ 13 \log(1/\gamma), \frac{100 \log (1/c)}{ \log 3}\right\}
\]

\end{proof}

\subsection{Failures in (0, 0, 0)-Unlearning Beyond Section \ref{sec:labelonly}} \label{sec:failures-in-unlearning}

Observable failures in unlearning guarantees for algorithms in the SISA framework go beyond the simplistic setting constructed in Section \ref{sec:labelonly}. In this section, we describe a more natural setting in which we employ the learning and unlearning algorithms for SISA $(\mathcal{A}, \mathcal{R_A})$ and are able to construct an adaptive deletion sequence (only given discrete predictions through $f_{\text{publish}}$) which, to a high degree of confidence, rejects the null hypothesis that $(\mathcal{A}, \mathcal{R_A})$ satisfy a perfect $(0, 0, 0)$-unlearning guarantee.

In Section \ref{sec:labelonly} we explicitly define a base model $f_\theta$ which relies on the fact that each point is copied twice to reveal perfect information about how points were partitioned through its predictions. Here, we define a new model which relaxes this condition. Given a query point $x$, rather than return the label of an exactly matching point, the model $f_{D, \tau}(x)$ is additionally parameterized by a threshold $\tau$. This model, reminiscent of 1-nearest neighbors, returns the label of the closest point $(x', y') \in D$ where $ |x - x'|_2 \leq \tau$, and $\bot$ otherwise, essentially treating nearby points as "identical."

Here we define $\cA^{\text{SISA}}$ and $\cR_{\cA^{\text{SISA}}}$ as Algorithms \ref{alg:learning-sisa} and \ref{alg:unlearning-sisa} respectively, instantiated with $\cAsingle(D) = D$ and prediction rule $f_{D, \tau}$. We assume the null hypothesis that $\cA$ and $\cR_{\cA}$ satisfy a $(0, 0, 0)$-unlearning guarantee.

To make an assertion about this hypothesis, we train an ensemble using three shards as before. We then execute a similar experiment to that as described in Section \ref{sec:labelonly} in which, after initial training, we publish the aggregated discrete predictions for each training point and delete a random subset of correctly classified points. We then observe the accuracy of the ensemble on the remaining training points. Our hypothesis, the same as before, is that the resulting accuracy will be lower in the adaptive deletion setting than the retrain setting with high probability.

We then define an event $E$ of interest to be when the training accuracy after the adaptive deletion sequence falls below a cutoff $c \in [0\%, 100\%]$ after deleting all correctly classified points. We can then estimate the probability of this event by defining an indicator for each trial which is 1 if the training accuracy falls below this threshold and 0 otherwise. We then run many trials to calculate confidence intervals on our estimate of this probability under either setting. If the confidence intervals are non-overlapping at some confidence level, we can then reject the null hypothesis at some level of confidence.

Our concrete experiment samples 1,000 random points from MNIST, each being either a "0" or "1" (preprocessing each image by dividing each pixel value by 255). With $\tau = 6.5$, this setting is "plausible" in the sense that this model's performance on held-out test data is nontrivial (approximately $91.2\%$ test accuracy before deletion) for a common benchmark task. We then delete $t$ points (a uniformly random subset of correctly predicted points), observe the average accuracies across trials on remaining points under the adaptive setting and the retrain scenario. We grid search for the $c$ which yields the largest difference in the confidence intervals (since the unlearning guarantee should hold for all $c$). Under these conditions we find that after 200 trials, we attain 97.5\% confidence intervals on our statistic to be those shown in Figure \ref{fig:labelonly}. We see that for deletion sequences of 200 points or more we can induce a reliable difference in this statistic at a high level of confidence, rejecting the null hypothesis at $p \leq 0.05$ that $\cA^{\text{SISA}}$ and $\cR_{\cA^{\text{SISA}}}$ satisfy a perfect $(0, 0, 0)$-unlearning guarantee.

\begin{figure}[t]
\centering
\includegraphics[width=8cm]{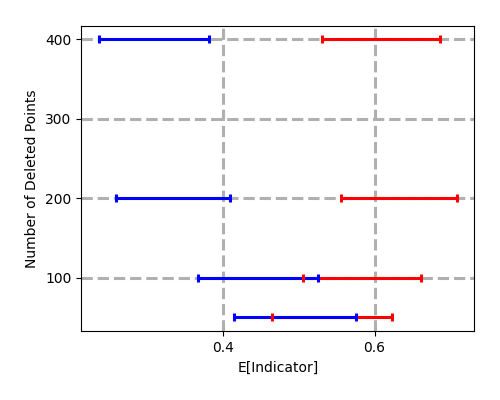}
\caption{Confidence intervals for the indicator defined in Section \ref{sec:failures-in-unlearning} as a function of the number of deleted points. Red confidence intervals correspond to the statistic after the adaptive deletion sequence, and blue confidence intervals correspond to the statistic after full retraining. For deletion sequences of 200 points or more we can induce a reliable enough difference in the confidence intervals to reject the null hypothesis at $p \leq 0.05$ that $\cA^{\text{SISA}}$ and $\cR_{\cA^{\text{SISA}}}$ satisfy a perfect $(0, 0, 0)$-unlearning guarantee in a realistic setting.}
\label{fig:labelonly}
\end{figure}

\subsection{Full Experiment Details of Section \ref{sec:fullmodelsetting}}

\begin{table}[h]
\centering
\begin{tabular}{lllll} \toprule
    {$\mathbb{E}$[Indicator]} & {Acc. (after)} & {Acc. (before)} & {Noise mult.} & {Shard pred. acc.} \\ \midrule
    
    $\text{CIFAR-10}_{k=6}$ \\
    \midrule
    $ [0.890, 0.952] $ & $ 0.507 \pm 0.025 $ & $ 0.572 \pm 0.011 $ & $    0    $ & $ 0.303 \pm 0.005     $ \\
    $ [0.784, 0.869] $ & $ 0.504 \pm 0.020 $ & $ 0.554 \pm 0.010 $ & $    0.15 $ & $ 0.257 \pm 0.004     $ \\
    $ [0.670, 0.772] $ & $ 0.487 \pm 0.021 $ & $ 0.525 \pm 0.011 $ & $    0.22 $ & $ 0.229 \pm 0.003     $ \\
    $ [0.490, 0.603] $ & $ 0.455 \pm 0.025 $ & $ 0.484 \pm 0.012 $ & $    0.3  $ & $ 0.205 \pm 0.003     $ \\
    \midrule
    
    $\text{CIFAR-10}_{k=2}$ \\
    \midrule
    $ [0.947, 0.987] $ & $ 0.448 \pm 0.033 $ & $ 0.521 \pm 0.019 $ & $    0    $ & $ 0.655 \pm 0.012     $ \\
    $ [0.797, 0.881] $ & $ 0.433 \pm 0.026 $ & $ 0.475 \pm 0.015 $ & $    0.2  $ & $ 0.587 \pm 0.007     $ \\
    $ [0.638, 0.744] $ & $ 0.419 \pm 0.025 $ & $ 0.452 \pm 0.015 $ & $    0.25 $ & $ 0.567 \pm 0.006     $ \\
    $ [0.493, 0.607] $ & $ 0.399 \pm 0.027 $ & $ 0.427 \pm 0.015 $ & $    0.3  $ & $ 0.550 \pm 0.005     $ \\
    \midrule 
    
    $\text{Fashion-MNIST}_{k=6}$ \\
    \midrule
    $ [0.819, 0.899] $ & $ 0.849 \pm 0.011 $ & $ 0.874 \pm 0.004 $ & $    0    $ & $ 0.248 \pm 0.007     $ \\
    $ [0.662, 0.765] $ & $ 0.838 \pm 0.011 $ & $ 0.854 \pm 0.005 $ & $    0.4  $ & $ 0.215 \pm 0.005     $ \\
    $ [0.540, 0.652] $ & $ 0.823 \pm 0.009 $ & $ 0.834 \pm 0.006 $ & $    0.6  $ & $ 0.198 \pm 0.004     $ \\
    $ [0.477, 0.590] $ & $ 0.810 \pm 0.012 $ & $ 0.820 \pm 0.006 $ & $    0.75 $ & $ 0.190 \pm 0.004     $ \\
    \midrule
    
    $\text{Fashion-MNIST}_{k=2}$ \\
    \midrule
    $ [0.976, 0.999] $ & $ 0.826 \pm 0.016 $ & $ 0.863 \pm 0.006 $ & $     0   $ & $ 0.597 \pm 0.014     $ \\
    $ [0.797, 0.881] $ & $ 0.808 \pm 0.013 $ & $ 0.828 \pm 0.007 $ & $     0.5 $ & $ 0.555 \pm 0.010     $ \\
    $ [0.607, 0.715] $ & $ 0.791 \pm 0.016 $ & $ 0.807 \pm 0.008 $ & $     0.7 $ & $ 0.538 \pm 0.007     $ \\
    $ [0.497, 0.610] $ & $ 0.763 \pm 0.020 $ & $ 0.781 \pm 0.009 $ & $     1   $ & $ 0.523 \pm 0.005     $ \\
    \midrule

    $\text{MNIST}_{k=6}$ \\
    \midrule
    $ [0.849, 0.922] $ & $ 0.973 \pm 0.004 $ & $ 0.978 \pm 0.002 $ & $     0   $ & $ 0.201 \pm 0.004     $ \\
    $ [0.729, 0.824] $ & $ 0.965 \pm 0.005 $ & $ 0.969 \pm 0.003 $ & $     0.4 $ & $ 0.186 \pm 0.003     $ \\
    $ [0.583, 0.694] $ & $ 0.940 \pm 0.009 $ & $ 0.945 \pm 0.004 $ & $     0.8 $ & $ 0.178 \pm 0.003     $ \\
    $ [0.493, 0.607] $ & $ 0.913 \pm 0.018 $ & $ 0.923 \pm 0.007 $ & $     1.1 $ & $ 0.176 \pm 0.003     $ \\
    \midrule
    
    $\text{MNIST}_{k=2}$ \\
    \midrule
    $ [0.927, 0.976] $ & $ 0.962 \pm 0.007 $ & $ 0.971 \pm 0.003 $ & $     0   $ & $ 0.540 \pm 0.006     $ \\
    $ [0.769, 0.857] $ & $ 0.959 \pm 0.008 $ & $ 0.968 \pm 0.003 $ & $     0.8 $ & $ 0.534 \pm 0.006     $ \\
    $ [0.587, 0.697] $ & $ 0.953 \pm 0.007 $ & $ 0.962 \pm 0.004 $ & $     1.3 $ & $ 0.530 \pm 0.006     $ \\
    $ [0.497, 0.610] $ & $ 0.949 \pm 0.008 $ & $ 0.957 \pm 0.004 $ & $     1.6 $ & $ 0.527 \pm 0.006     $ \\
    \midrule
\end{tabular}
\caption{Numerical representation of results displayed in Figure \ref{fig:exp-figs}. The $x$ axis in Figure \ref{fig:exp-figs} corresponds to column "$\mathbb{E}[\text{Indicator}]$", and the $y$ axis corresponds to column "Acc. (after)". Column "$\mathbb{E}[\text{Indicator}]$" represents the 95\% confidence interval of the indicator after 300 trials. Columns "Acc. (before)" and "Acc. (after)" represent the accuracy of the ensemble on a held-out test set (5,000 points each) before and after deleting approximately half of the points from the ensemble, with confidence intervals given by two standard deviations above and below the observed mean. "Noise multiplier" represents the standard deviation of Gaussian noise applied to each per-example gradient during DP-SGD. Shard prediction accuracy denotes the prediction accuracy of the adversary in targeting models when deleting points, where random guessing would achieve an accuracy of $1 / (\text{\# shards})$.}
\end{table}

Choices in hyperparameters and and model architecture for experiments presented in Section \ref{sec:experiment} were inspired by those used by \cite{temperedsigmoid}. All models were optimized using momentum with mass equal to 0.9. The clipping parameter (upper bound on maximum $\ell_2$-norm of per-example gradients) used in DP-SGD for all experiments was equal to 0.1. For certain experiments, the batch size was reduced from what was presented in \cite{temperedsigmoid} to reduce computational cost. Each experiment was repeated with new random seeds across 300 trials to get the confidence intervals displayed in Figure \ref{fig:exp-figs}. The precise model definition for each experiment is given below:

\begin{table}
\centering
\begin{tabular}{lcccccc} \toprule
Experiment                     & Points per shard  & Batch Size & Iterations & Step size    \\ \midrule
$\text{CIFAR-10}_{k=6}$        & 8000              & 64         & 4000    & 1.0          \\
$\text{CIFAR-10}_{k=2}$        & 8000              & 64         & 4000    & 1.0          \\
$\text{Fashion-MNIST}_{k=6}$   & 6000              & 256        & 1500    & 4.0          \\ 
$\text{Fashion-MNIST}_{k=2}$   & 6000              & 256        & 2000    & 4.0          \\ 
$\text{MNIST}_{k=6}$           & 6000              & 64         & 2500    & 0.5          \\
$\text{MNIST}_{k=2}$           & 6000              & 256        & 2000    & 0.5          \\
\midrule
\end{tabular}
\caption{Remaining hyperparameter settings for each experiment, by dataset.}
\end{table}

\begin{verbatim}
Sequential(
  Conv(out_chan=16, filter_shape=(8, 8), padding='SAME', strides=(2, 2)),
  Tanh,
  MaxPool(window_shape=(2, 2), strides=(1, 1)),
  Conv(out_chan=32, filter_shape=(4, 4), padding='VALID', strides=(2, 2)),
  Tanh,
  MaxPool(window_shape=(2, 2), strides=(1, 1)),
  Flatten,
  Dense(out_dim=32),
  Tanh,
  Dense(out_dim=num_classes)
)
\end{verbatim}

In our experiments we make use of 3 common benchmark machine learning datasets. The MNIST database of handwritten digits given by \cite{mnist} consists of 70,000 28 $\times$ 28 images of handwritten digits, each belonging to one of 10 classes characterizing the digit shown in each image. MNIST is made available under the Creative Commons Attribution-Share Alike 3.0 license. The Fashion-MNIST dataset given by \cite{fmnist} consists of 70,000 28 $\times$ 28 grayscale images of pieces of clothing, each belonging to one of 10 classes (e.g. t-shirt, dress, sneaker, etc.) Fashion-MNIST is made available under the MIT license. The CIFAR-10 dataset given by \cite{cifar10} consists of 60,000 32 $\times$ 32 images in RGB format, each belonging to one of 10 classes characterizing the class of the object given in each image (e.g. airplane, automobile, bird, etc.) CIFAR-10 is made available under the MIT license. 

With respect to computing environment, experiments were conducted using the JAX deep learning framework developed by \cite{jax}. Experiments were run using 1 Tesla V100 GPU using CUDA version 11.0, where an individual trial (training a full ensemble, deleting targeted points, and retraining) would take approximately 1-6 minutes depending on the number of shards, iterations, image size, etc.

\end{document}